\documentclass[twocolumn]{article} 
\usepackage[top=1.5in, bottom=1.in, left=0.5in, right=0.5in]{geometry}
\usepackage[utf8]{inputenc} 
\usepackage{authblk} 
\usepackage{amsfonts,amsmath,amsthm,amssymb}
\usepackage{hyperref}
\usepackage{graphicx}
\usepackage{algorithm} 
\usepackage{algorithmic} 
\usepackage[style=numeric, backend=biber]{biblatex}

\usepackage{bm} 
\usepackage{color} 

\usepackage{tikz}
\usetikzlibrary{decorations.pathmorphing} 
\usetikzlibrary{matrix} 
\usetikzlibrary{arrows} 
\usetikzlibrary{calc} 

\tikzstyle{block} = [draw,rectangle,thick,minimum height=2em,minimum width=2em]
\tikzstyle{sum} = [draw,circle,inner sep=0mm,minimum size=2mm]
\tikzstyle{connector} = [->,thick]
\tikzstyle{line} = [thick]
\tikzstyle{branch} = [circle,inner sep=0pt,minimum size=1mm,fill=black,draw=black]
\tikzstyle{guide} = []
\tikzstyle{snakeline} = [connector, decorate, decoration={pre length=0.2cm,
                         post length=0.2cm, snake, amplitude=.4mm,
                         segment length=2mm},thick, magenta, ->]

\newtheorem{proposition}{Proposition}

\DeclareMathOperator{\tr}{tr}
\newcommand\numberthis{\addtocounter{equation}{1}\tag{\theequation}}
\newcommand{\pose}{X}
\renewcommand{\angle}{\theta}
\newcommand{\state}{\mathbf{x}}
\newcommand{\measurement}{\mathbf{z}}
\newcommand{\command}{\mathbf{u}}
\newcommand{\noisemodel}{\mathbf{m}}
\newcommand{\noisemeas}{\mathbf{n}}
\newcommand{\local}[1]{#1^{loc}}

\newcommand{\reference}[1]{#1^*}
\newcommand{\disttoref}[1]{\bar{#1}}
\newcommand{\disttoestimate}[1]{\tilde{#1}}
\newcommand{\estimate}[1]{\hat{#1}}

\newcommand{\normal}{{\mathcal{N}}}

\title{An Invariant Linear Quadratic Gaussian controller for a simplified car}
\author[1]{Sébastien Diemer}
\author[1]{Silvère Bonnabel}
\affil[1]{MINES ParisTech, PSL - Research University, Center for robotics, 60 Bd St Michel 75006 Paris, France, \texttt{[sebastien.diemer, silvere.bonnabel]@mines-paritech.fr}}
\begin{document}
\maketitle
\begin{abstract}
In this paper, we consider the problem of tracking a reference trajectory for a simplified car model based on unicycle kinematics, whose position only is measured, and where the control input and the measurements are corrupted by independent Gaussian noises. To tackle this problem we devise a novel observer-controller: the invariant Linear Quadratic Gaussian controller (ILQG). It is based on the Linear Quadratic Gaussian controller, but the equations are slightly modified to account for, and to exploit, the symmetries of the problem. The gain tuning exhibits a reduced dependency on the estimated trajectory, and is thus less sensitive to misestimates. Beyond the fact the invariant approach is sensible (there is no reason why the controller performance should depend on whether the reference trajectory is heading west or south), we show through simulations that the ILQG outperforms the conventional LQG controller in case of large noises or large initial uncertainties. We show that those robustness properties may also prove useful for motion planning applications.  
\end{abstract}
\section{Introduction}
\label{introduction}
The field of mobile robot control has been thoroughly studied in the past.
One class of problems of practical interest is the trajectory tracking problem, for which the robot's goal is to follow a predefined time-parameterized path.
The Linear Quadratic Gaussian (LQG) controller is a standard tool of linear  control that can handle process and measurement Gaussian noises and possesses optimality properties, see e.g. \cite{Stengel12}.
It consists of coupling a Kalman filter for the state estimation, and a Linear Quadratic (LQ) controller for the trajectory tracking. When the system is not linear, as this is the case for the unicycle model due to the cosines and sines terms, the LQG can be extended (at the price of a linearization of the model and output equations about the reference trajectory) but the (extended) LQG controller looses all its optimality properties. In particular, as soon as the true trajectory tends to deviate from the reference trajectory (due to a perturbation, that may originate from perceptual ambiguity for instance, or a large initial uncertainty, or merely large noises) there is no guarantee at all the LQG should be able to drive the robot back to the reference trajectory.  

In a deterministic setting, there has been numerous attempts to account for the symmetries of the problem in controller design (the geometric control literature being extremely vast and dating back to \cite{sussmann-72}). Notably, the interested reader is referred to the more recent work \cite{bullo-murray-auto99} where nonlinear controllers are devised for a relevant class of simple mechanical systems.  As concerns state estimation, there as been an increasing number of attempts to account for the symmetries in observer design over the last decade, the major body of literature on the subject having been motivated by attitude estimation (to cite a couple of papers see \cite{mahony-et-al-IEEE,arxiv-07}). The idea of combining invariant state estimation (that is, the art of building estimators that respect the symmetries of the problem) and invariant control can be traced back to \cite{Guillaume98}  which devises an observer-controller for the same problem as the one considered in the present paper, that is, trajectory tracking for a simplified car whose position only is measured. Finally, still in a deterministic setting, the work \cite{bonnabel-et-al:ifac11} discusses a separation principle for invariant observer-controllers.

In a stochastic setting, where noises are to be explicitly taken into account, the invariant Extended Kalman Filter (IEKF) introduced in continuous time in \cite{bonnabel_cdc07,bonnabel2009invariant} and in discrete time in \cite{barrau-bonnabel-cdc13}, is a novel methodology that aims at modifying the equations of the Extended Kalman Filter (EKF) a little so that they respect the symmetries of the problem (see also the more recent work of \cite{barczyk2013invariant}). The IEKF possesses some convergence properties that the EKF lacks. The goal of the present paper is to combine an IEKF with a LQ controller that respects the symmetries of the system under rotation and translation, in order to have a simplified car track a predefined trajectory in the presence of Gaussian noises. For this system, the invariant approach boils down to considering both the estimation error and the tracking error in the Frénet coordinates, that is, a moving frame attached to the car, and to devise a Kalman filter and a LQ controller for stabilizing the linearized errors. To this respect, the invariant LQ controller  can be also related to the work \cite{rouchon-rudolph-ncn99} that advocates the use of Frénet coordinates for tracking in SE(2). 

The  introduced invariant LQG controller possesses several properties. First of all, it is invariant to rotations and translations, that is, it ensures that the behavior is the same whether the car is, for example, trying to park automatically along a North-oriented of West-oriented sidewalk. Surprisingly enough, this property does not generally hold when devising a standard LQG controller for this problem. Moreover, the Kalman gain is proved to be independent of the estimated trajectory, a property which reminds the linear case. Indeed, in extended Kalman filtering, the equations being linearized about the \emph{estimated} trajectory, an erroneous estimate of the state can lead to a inappropriate gain that can in turn generate an even more erroneous estimate. This kind of positive feedback can lead to divergence of the filter, and  cannot occur when dealing with the IEKF proposed herein. Without exploring the theory of IEKF stability, which goes beyond the scope of the present paper, we show through extensive simulations the robustness of the proposed invariant LQG versus conventional LQG.

The remainder of this article is organized as follows. In \autoref{sec:LQGapproach} we define formally the problem at stake, the robot kinematics and its environment, and the conventional LQG traditionally used for trajectory tracking. Then, we derive the  equations of  the invariant LQG in \autoref{sec:invariant} and review its basic properties. In \autoref{sec:simulation_LQG} we compare through simulations the performances of the proposed invariant LQG with those of the conventional LQG. Simulations show the invariant approach outperforms the conventional one in case of large noises or large uncertainty on the initial state. Finally, we adapt in \autoref{sec:apriori} an approach recently introduced in \cite{VanDenBerg11p895}, and show how the linearized equations of the invariant observer and the invariant controller can be combined to determine in advance the probability distributions of the state of the robot along the reference trajectory.
This information can be used, for instance, to evaluate the probability of success of a planned trajectory. In this respect, it can be used by a planner to explicitly account for sensors and control uncertainties as do the planners from e.g. \cite{Roy99p35, Tedrake10p1038}.
Simulations indicate that the computed probability distributions capture much more closely the true dispersion of the tracking error (obtained through Monte-Carlo simulations) when an invariant LQG is used rather than a conventional LQG. 

\section{Problem formulation and LQG controller}
\label{sec:LQGapproach}
\subsection{Problem formulation}
\label{ssec:problemformulation}
In this paper, we consider a non-holonomic unicycle robot (simple car model) moving in a two-dimensional world (see e.g. \cite{campion1996robotics}. The robot is characterized by its state $\state = (\pose , \theta) \in \chi \subset \mathbf{R}^3$ where $\pose = (x, y)$ is the robot's position and $\theta$ its orientation. The dynamics governing the update of the state $\state_t$ to $\state_{t+1} = f(\state_t, \command_t, \noisemodel)$ writes:
\begin{align*}
    x_{t+1} &= x_t + \tau(u_t + v)\cos(\theta_t) \\
    y_{t+1} &= y_t + \tau(u_t + v)\sin(\theta_t) \numberthis \label{eq:dynamics}\\
    \theta_{t+1} &= \theta_t + \tau(\omega_t + w)
\end{align*}
where 
$\tau$ is the discretized time step, $\command = (u, \omega)$ is the system inputs, and $\noisemodel = (v, w)$ is the model noise.
The robot has access to its absolute pose in the environment through for instance a GPS or a video tracking system, yielding measurements $\measurement = H \state = (x + \noisemeas_x, y + \noisemeas_y).$
The orientation $\theta$ is supposed not to be measured.
We suppose that both the motion and the measurement noises are white and Gaussian.
$$\noisemodel \sim \normal(0, M), \qquad \noisemeas \sim \normal (0, N)$$
The noises $\noisemodel$ and $\noisemeas$ at all time steps $t$ are assumed to be mutually independent.
Finally, we suppose that $\noisemeas$ is isotropic (i.e. $N = \lambda I_2$), a reasonable assumption for GPS measurements restricted to an horizontal plane.

The environment contains a collection of obstacles $\chi^{obs}$ that  the robot must avoid colliding.
We denote $\chi^{free}=\chi \setminus \chi^{obs}$ the free space and $\chi^{goal} \subset \chi^{free}$ the goal region the robot must reach.
We define the reference trajectory as a collection of states $\reference \state_{0} , \ldots, \reference \state_n$ where $\reference \state_0 = \state^{start}$, $\reference \state_n \in \chi^{goal}$,  and $\forall t \quad 0 \leq t < n, \quad \reference \state_{t+1} = f(\reference \state_t, \reference \command_t, 0)$ that is to say, an evolution governed only by the dynamics and without motion noise. We denote by $\disttoref \state = \state - \reference \state$ and $\disttoref \command = \command - \reference \command$ the errors between the true and the reference trajectory.

In order for the robot to stay near the reference trajectory despite of the uncertainties in the measurements and the controls, we design a linear-quadratic controller, which aims at minimizing, under the dynamics constraints of \eqref{eq:dynamics}, the cost function $J(\state, \command):$ 
\begin{equation}
    J(\state, \command) = \mathbb E\left( I(\state, \command)\right)=
    \mathbb E\left(\sum_{t=0}^n (\disttoref \state_t C \disttoref \state_t^T
    + \disttoref \command_t D \disttoref \command_t^T) \right)
    \label{eq:cost}
\end{equation}
with $C$ and $D$ definite positive matrices that penalize the deviations in the tracking and in the actuator's command. $J$ thus appears as the average over a great number of experiments of the overall deviation $I$ associated to a single trajectory.

\subsection{Conventional LQG}
One approach to attack the problem defined in Section \ref{ssec:problemformulation}, is to linearize $f$ around the reference trajectory and use a conventional Linear Quadratic Gaussian (LQG) control.
LQG combines a Kalman filter for state estimation and a Linear Quadratic controller for the control.
It provides an optimal control, which minimizes the cost \eqref{eq:cost} in the case of linear dynamics (see e.g. \cite{Stengel12}).
The conventional LQG Algorithm is recalled in Algorithm \ref{alg:LQG}.

\begin{algorithm}[h]
    \caption{Conventional LQG}
    \label{alg:LQG}
    \begin{algorithmic}
        \REQUIRE Reference trajectory $\left( (\state_t, \command_t) \right)_{t=1\dots n}$
        \REQUIRE Initial covariance $P_0$
        \REQUIRE Off-line calculations of the Riccati gains \eqref{eq:LQ_Riccati_gain}
        \STATE{$\command_0 \leftarrow \reference \command_0$}
        \STATE{$\estimate \state_0 \leftarrow \reference \state_0$}
        \FOR{$0 < t \leq n$}
        \STATE{
            \begin{itemize}
                \item propagate estimation and covariance with \eqref{eq:Kalman_pred_est} and \eqref{eq:Kalman_pred_cov} \\
                \item acquire measurement $\measurement_t$\\
            \item update the best estimate and the covariance using \eqref{eq:Kalman_updt_est} and \eqref{eq:Kalman_updt_cov}\\
                \item output new command $\command_t$ computed by \eqref{eq:control}
            \end{itemize}
        }
        \ENDFOR
    \end{algorithmic}
\end{algorithm}

The best estimate (in the sense of least squares over a great number of experiments and under the linear approximation), and its covariance updates are given by the conventional extended Kalman filter (EKF) equations: 
\\

\emph{Process update (conventional EKF):}
\begin{align}
    P^-_{t + 1} &= A_t P_t A_t^T + B_t M B_t^T \label{eq:Kalman_pred_cov}\\
    \estimate \state^-_{t + 1} &= f(\estimate \state_t, \command_t, 0) \label{eq:Kalman_pred_est}
\end{align}
where here:
\begin{align*}
    A_t &= \frac{\partial f}{\partial \state}\bigg|_{\estimate \state_t, \command_t, 0} =
    \begin{pmatrix}
     1 & 0 & -\tau u_t \sin(\estimate \theta_t) \\ 
     0 & 1 & \tau  u_t \cos(\estimate \theta_t) \\ 
     0 & 0 & 1
    \end{pmatrix}, \\
    B_t &=
    \frac{\partial f}{\partial \noisemodel}\bigg|_{\estimate \state_t, \command_t, 0} =
    \tau
    \begin{pmatrix}
     \cos(\estimate \theta_t) & 0 \\
     \sin(\estimate \theta_t) & 0 \\
     0 & 1 
    \end{pmatrix}
\end{align*}

\emph{Measurement update (conventional EKF):}
 \begin{align}
    K_{t+1} &= P^-_{t+1} H^T ( H P^-_{t+1} H^T + N_t )^{-1} \\
    P_{t+1} &= (I-K_{t+1} H)P^-_{t+1} \label{eq:Kalman_updt_cov}\\
    \estimate \state_{t + 1} &= \estimate \state^-_{t+1}
    + K_{t+1} (\measurement_t - H \estimate \state^-_{t+1}) \label{eq:Kalman_updt_est}
\end{align}
Likewise, the LQ controller linearized equations read:
\begin{equation}
    \disttoref \state_{t+1} =
    A^*_t \disttoref \state_t + 
    B^*_t \disttoref \command_t + 
    B^*_t \noisemodel 
    \label{eq:controller_conventional_LQ}
\end{equation}
where here:
\begin{align*}
    A^*_t &= \frac{\partial f}{\partial \state}\bigg|_{\reference \state_t, \reference \command_t, 0} =
    \begin{pmatrix}
     1 & 0 & -\tau u_t \sin(\reference \theta_t) \\ 
     0 & 1 & \tau  u_t \cos(\reference \theta_t) \\ 
     0 & 0 & 1
    \end{pmatrix} \\
    B^*_t &=
    \frac{\partial f}{\partial \command}\bigg|_{\reference \state_t, \reference \command_t, 0} =
    \frac{\partial f}{\partial \noisemodel}\bigg|_{\reference \state_t, \reference \command_t, 0} =
    \tau
    \begin{pmatrix}
     \cos(\reference \theta_t) & 0 \\
     \sin(\reference \theta_t) & 0 \\
     0 & 1 
    \end{pmatrix}
\end{align*}
and the updated control law reads: 
\begin{equation}
    \command_{t}=\reference \command_t + L_t (\estimate \state_t - \reference \state_t)
    \label{eq:control}
\end{equation}
where the gains $L_t$ are computed through the following backwards Riccati equation \eqref{eq:LQ_Riccati_gain}:

\begin{align}
    S_l &= C \\
    L_t &= -({B^*_t}^T S_t B^*_t + D)^{-1} {B^*_t}^T S_t A^*_t \\ 
    S_{t} &= C + {A^*_{t+1}}^T S_{t+1} A^*_{t+1} + {A^*_{t+1}}^T S_{t+1} B^*_{t+1} L_{t+1} 
    \label{eq:LQ_Riccati_gain}
\end{align}

One noticeable characteristic of the conventional LQG is that the linearized matrices depend on the trajectory through the estimated orientation $\estimate \theta_t$ for the observer, and the reference orientation $\reference \theta_t$ for the controller. This feature is illustrated on the diagram of Figure \ref{fig:ekf_diagram}. This is in sharp contrast with the case of linear systems, and might be a cause of divergence of the closed-loop system. In the next section, we design an invariant LQG observer-controller, for which the linearized matrices only depend on the inputs, and this will be shown to increase the robustness compared to the conventional approach.

\begin{figure}[h]
    \centering
    \includegraphics[width=0.8\linewidth]{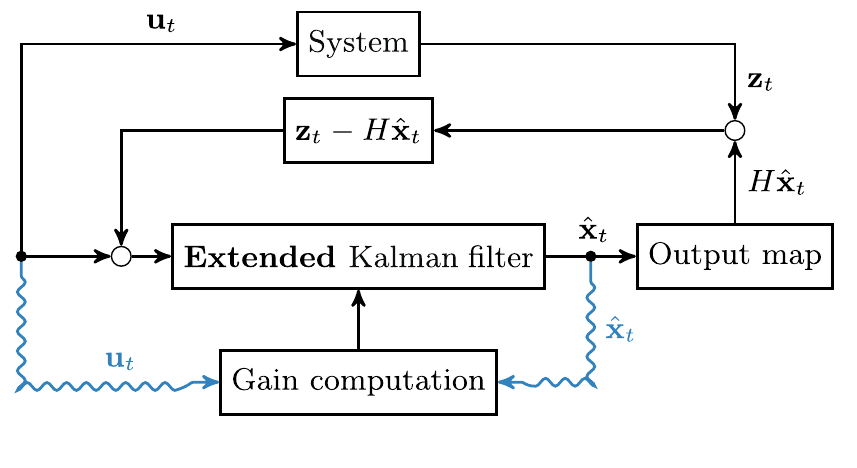}
    \caption{The gain computation of the EKF depends directly on the last estimate.}
    \label{fig:ekf_diagram}
\end{figure}

\section{The invariant LQG controller}
\label{sec:invariant}
In this section we define a LQG observer-controller with invariant properties.
The design builds upon the following remark \cite{Guillaume98}: the dynamics \eqref{eq:dynamics} are invariant to rotations and translations, that is, they do not depend on the choice of frame.
\begin{proposition}
    The dynamics \eqref{eq:dynamics} are invariant to rotations and translations.
    Consequently the LQG equations for \eqref{eq:dynamics} do not depend on the choice of coordinates in the following sense:
let $\left( \begin{smallmatrix} x \\ y \\ \theta \end{smallmatrix} \right) = \left( \begin{smallmatrix} x^0 \\ y^0 \\ 0 \end{smallmatrix} \right) + \left( \begin{smallmatrix} X\cos\theta^0 - Y\sin\theta^0 \\ Y\sin\theta^0 + X\cos\theta^0 \\ \Theta + \theta_0 \end{smallmatrix} \right)$ be a change of coordinates from a reference system of coordinates $\left( \begin{smallmatrix} X \\ Y \\ \Theta \end{smallmatrix} \right)$ to a rotated and translated system of coordinates $\left( \begin{smallmatrix} x \\ y \\ \theta \end{smallmatrix} \right)$, then the equations \eqref{eq:dynamics} write the same when written with the transformed variables.\end{proposition}
\begin{proof}
\footnotesize
\begin{align*}
\text{If}
&\begin{cases}
    X_{t+1} = X_t + \tau(u_t + v)\cos(\Theta_t) \\
    Y_{t+1} = Y_t + \tau(u_t + v)\sin(\Theta_t) \\
    \Theta_{t+1} = \Theta_t + \tau(\omega_t + w)
\end{cases}
\text{Then} \\
&\begin{cases}
    x_{t+1}  \triangleq X^0 + X_{t+1} \cos(\Theta^0) - Y_{t+1} \sin(\Theta^0) \\
    y_{t+1}  \triangleq Y^0 + X_{t+1} \sin(\Theta^0) + Y_{t+1} \cos(\Theta^0) \\
    \theta_{t+1}  \triangleq \Theta^0 + \Theta_{t+1}
\end{cases} \\
\Leftrightarrow &\begin{cases}
    x_{t+1} = X^0 + X_t \cos(\Theta^0) - Y_t \sin(\Theta^0) + \tau(u_t+v) \cos(\Theta_t + \Theta^0) \\
    y_{t+1} = Y^0 + X_t \sin(\Theta^0) + Y_t \cos(\Theta^0) + \tau(u_t+v) \sin(\Theta_t + \Theta^0)\\
    \theta_{t+1} = \Theta^0 + \Theta_t + \tau(\omega_t + w)
\end{cases} \\
\Leftrightarrow &\begin{cases}
    x_{t+1} = x_t + \tau(u_t + v) \cos(\theta_t) \\
    y_{t+1} = y_t + \tau(u_t + v) \sin(\theta_t) \\
    \theta_{t+1} = \theta_t + \tau(\omega_t + w)
\end{cases}
\end{align*}
\normalsize
\end{proof}
It seems evident that the ability of an observer-controller to park along the pavement should not depend on whether the pavement is north-oriented or west-oriented. However, surprisingly, the conventional LQG controller \eqref{eq:control} has a behaviour that does depend on orientation of the car ($\reference \theta_t, \estimate \theta_t$). When designing both an observer and a controller, the problem can be remedied by deriving  the observer and the controller equations in the Frénet coordinates (see \cite{rouchon-rudolph-ncn99,arxiv-07}). In the remainder of the paper, we use the superscript $\square^{loc}$ (loc stands for local) to identify the vectors expressed in the Frénet frame, that is, a frame attached to the car whose first axis coincides with the car's heading direction $\theta_t$.

\subsection{Invariant Extended Kalman Filter (IEKF)}
\label{sec:invariantobserver}

To estimate the state $\state$ of the robot, we use an invariant formulation of the extended Kalman filter as proposed in \cite{bonnabel_cdc07,bonnabel2009invariant}. Here it boils down to working in the Frénet frame as follows. 
Let us define the estimation error $\disttoestimate \state = \estimate \state - \state$, the local estimation error
$\local{\disttoestimate{\state}} = \Upsilon_{-\angle} \disttoestimate{\state}$,
and the local state deviation $\local{\disttoref{\state}} = \Upsilon_{-\reference{\angle}} \disttoref{\state}$, where 
$ \Upsilon_{\phi} = \left(
\begin{smallmatrix}
    R_{\phi} & 0 \\ 0 & 1
\end{smallmatrix} \right)$.
We can note that, by definition, $\Upsilon$ has the following properties\footnote{Note that those properties can be related to the theory of symmetry groups \cite{olver-book95}. Yet, in the present paper, we prefer to keep calculations at a basic level to remain close to the computer implementation.} that will be used in the sequel: $\Upsilon_0 = I_3$, $\Upsilon_{\phi + \psi} =\Upsilon_\phi \Upsilon_\psi$ and $H \Upsilon_\phi = R_\phi H$. 

We search to estimate the state of the robot using a filter of the following form \cite{arxiv-07,bonnabel_cdc07}:
\begin{align}
    \estimate \state^-_{t+1} &= f(\estimate \state_t, \command_t, 0) \label{eq:Kalman_pred_est_inv} \\
    \estimate \state_{t+1} &= \estimate \state^-_{t+1} +
    \Upsilon_{\estimate \theta^-_{t+1}} K^{inv}_t R_{-\estimate \theta^-_{t+1}}
    (\measurement_{t+1} - H \estimate \state^-_{t+1}) 
    \label{eq:invariant_estimate}
\end{align}
where $R_{-\estimate \theta^-_{t+1}}$ represents the 2D rotation of angle $-\estimate \theta^-_{t+1}$ (the opposite of the third coordinate of $\estimate \state^-_{t+1}$).
The idea behind the proposed filter is merely to map the measurement error $z_{t+1} - H\estimate \state^-_{t+1}$ into the Frénet frame of the estimated car, that is, applying a rotation of angle $-\estimate \theta$, then apply the Kalman correction gain $K$, and finally map the obtained correction term back into the inertial frame through the operator $\Upsilon_{\estimate \theta}$.

The evolution of the local estimation error writes:
\begin{align*}
    \local{\disttoestimate \state_{t+1}} =
    &\Upsilon_{- \theta_{t+1}}(\estimate \state_{t+1} - \state_{t+1}) \\
    =&\Upsilon_{-\theta_{t+1}}
    \Big( f(\estimate \state_t, \command_t, 0) - f(\state_t, \command_t, \noisemodel) +  \\
    &\Upsilon_{\estimate \theta^-_{t+1}} K^{inv}_{t+1} R_{-\estimate \theta^-_{t+1}}
    \left(\measurement_{t+1} - Hf(\estimate \state_t, \command_t, 0)\right)\Big) \\
    \approx&\Upsilon_{-\theta_{t+1}}
    \frac{\partial f}{\partial \state}
    \bigg|_{\estimate \state_t, \command_t, 0}
    \disttoestimate \state_t \\
    &-\Upsilon_{-\theta_{t+1}}
    \frac{\partial f}{\partial \noisemodel}
    \bigg|_{\estimate \state_t, \command_t, 0} \noisemodel \\
    &+\Upsilon_{\disttoestimate \theta_t - \tau w} K^{inv}_t R_{-\estimate \theta_t -\tau \omega_t}
    \left(\measurement_{t+1} - Hf(\estimate \state_t, \command_t, 0)\right) \\
\end{align*}

For sufficiently small $\tau$ and considering noises as first order terms, we have up to second order terms:
\begin{align*}
        \Upsilon_{-\theta_{t+1}}\frac{\partial f}{\partial \state}\bigg|_{\estimate \state_t, \command_t, 0} &= 
        \Upsilon_{-\tau (\omega_t + w)} \Upsilon_{-\theta_t} \frac{\partial f}{\partial \state}\bigg|_{\estimate \state_t, \command_t, 0} \\
        &= \Upsilon_{-\tau (\omega_t + w)}
        \begin{pmatrix}
            \cos \theta_t & \sin \theta_t & 0 \\
            -\sin \theta_t & \cos \theta_t & \tau u_t \\
            0 & 0 & 1
        \end{pmatrix} \\
        &= \begin{pmatrix}
                 1 & \tau \omega_t & 0 \\
                 -\tau \omega_t & 1 & \tau u_t \\
                 0 & 0 & 1
            \end{pmatrix}\Upsilon_{-\theta_t} \\
        &= A_t \Upsilon_{-\theta_t} \\
        \Upsilon_{-\theta_{t+1}}\frac{\partial f}{\partial \noisemodel}\bigg|_{\estimate \state_t, \command_t, 0} &= 
        \Upsilon_{-\tau (\omega_t + w)} \Upsilon_{-\theta_t} \frac{\partial f}{\partial \noisemodel}\bigg|_{\estimate \state_t, \command_t, 0} \\
        &= \tau
        \begin{pmatrix}
            1 & 0 \\
            0 & 0 \\
            0 & 1
        \end{pmatrix}\\
        &= B
\end{align*}
Likewise, we have up to second order terms:
\begin{align*}
    &\Upsilon_{\disttoestimate \theta_t - \tau w}K^{inv}_t R_{-\estimate \theta_{t} - \tau \omega_t} \left(\measurement_{t+1} - Hf(\estimate \state_t, \command_t, 0)\right) \\
    = &K^{inv}_t H \Upsilon_{-\estimate \theta_{t} - \tau \omega_t}\left(
        -\frac{\partial f}{\partial \state}\bigg|_{\estimate \state_t, \command_t, 0} \tilde \state_t
    +\frac{\partial f}{\partial \noisemodel}\bigg|_{\estimate \state_t, \command_t, 0} \noisemodel \right) \\
    &+ K^{inv}_t \noisemeas_{t+1} \\
    = &K^{inv}_t H \left(-A_t \Upsilon_{-\theta_t} \disttoestimate \state_t + B \noisemodel \right) + K^{inv}_t \noisemeas_{t+1}
\end{align*}

Finally, up to second order terms, $\local{\disttoestimate \state_t}$ follows the linear evolution:
\begin{align} \label{eq:invariant_observer}
    \local{\disttoestimate \state}_{t+1} = &A_t \local{\disttoestimate \state}_t - B \noisemodel \nonumber \\
    &- K^{inv}_t H \left(A_t \local{\disttoestimate \state}_t - B \noisemodel \right) +  K^{inv}_t \noisemeas_t
\end{align}
with:
$$A_t = 
        \begin{pmatrix}
             1 & \tau \omega_t & 0 \\
             -\tau \omega_t & 1 & \tau u_t \\
             0 & 0 & 1
        \end{pmatrix}, \quad
B = \tau
    \begin{pmatrix}
        1 & 0 \\
        0 & 0 \\
        0 & 1
    \end{pmatrix}$$
and where with a slight abuse of notation we replaced $R_{-\theta_t}\noisemeas_t$ with $\noisemeas_t$ due to the measurement noise isotropy.
We thus proved, that the invariant linearized estimation error $\local{\disttoestimate \state_t}$ follows a linear equation for which the optimal gain $K^{inv}_t$ is given by the Kalman updates:\\
\emph{Process update (invariant Kalman):}
\begin{equation}
    P^-_{t + 1} = A_t P_t A_t^T + B M B^T \label{eq:Kalman_pred_cov_inv}
\end{equation}
\emph{Measurement update (invariant Kalman):}
 \begin{align}
     K^{inv}_{t+1} &= P^{loc-}_{t+1} H^T ( H P^{loc-}_{t+1} H^T + N_t )^{-1} \label{eq:Kalman_gain_inv}\\
    \local P_{t+1} &= (I-K^{inv}_{t+1} H)P^{loc-}_{t+1} \label{eq:Kalman_updt_cov_inv}
\end{align}
Finally, the invariant Kalman estimate is given by equation \eqref{eq:invariant_estimate} where $K^{inv}_t$ is computed by the above formula \eqref{eq:Kalman_gain_inv}.
We can notice that, as a byproduct of making use of the symmetries of the problem, we derived a linearized equation for the observer in the Frénet coordinates, in which $B$ does not depend on the trajectory at all, whereas $A_t$ only depends on the control inputs:
\begin{proposition}
    The linearized equation of the invariant Kalman filter in the Frénet coordinates is:
    $\local{\disttoestimate \state}_{t+1} = A_t \local{\disttoestimate \state}_t - B \noisemodel - K^{inv}_t H \left(A_t
    \local{\disttoestimate \state_t} - B \noisemodel \right) + K^{inv}_t \noisemeas_t$ where $B$ and $A_t=A(\command_t)$ do not depend on the state $\estimate\state_t$.
    \end{proposition}
    The property is illustrated by the diagram of Figure \ref{fig:iekf_diagram}.
    We will show experimentally in Section \ref{sec:simulation_LQG} that this property will endow the invariant LQG with better robustness to high noises and erroneous initialization, compared to the conventional LQG. This can be easily  understood as the gain output by the IEKF around \emph{any} trajectory is the same as the one output about the \emph{true} trajectory. Although it does not ensure the gain is optimal, it prevents  the type of divergences due to  a positive feedback between a misestimate and an inappropriate gain as explained in the Introduction. 

\begin{figure}[h]
    \centering
    \includegraphics[width=0.8\linewidth]{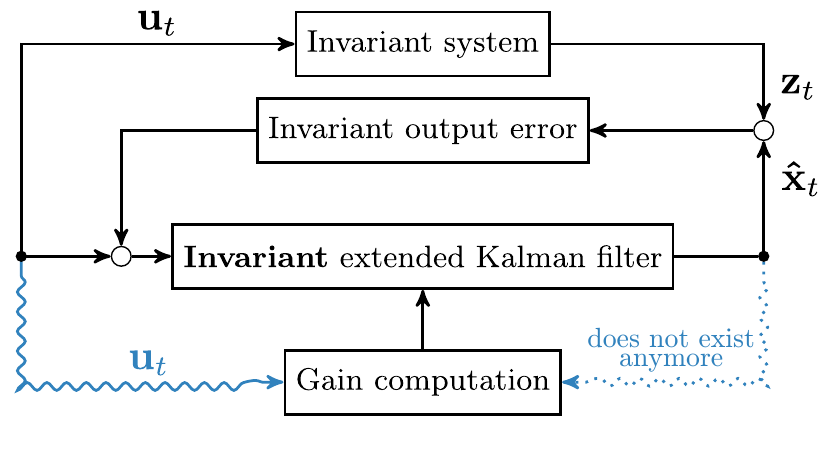}
    \caption{Contrarily to the EKF, the gain computation does not directly depend on the last estimate of the state in the case of the IEKF.}
    \label{fig:iekf_diagram}
\end{figure}

\subsection{Invariant LQ}
\label{sec:invariantcontroller}

Likewise, we can rewrite the dynamics of the reference trajectory error $\disttoref \state$ in the Frénet coordinates, to get a linearized invariant formulation of the LQ controller's equations.
We define the local error to the reference $\local{\disttoref \state} = \Upsilon_{-\reference \theta} \disttoref \state$ and differentiate, neglecting the second order terms:
\begin{align*}
        \local{\disttoref \state}_{t+1} = &\Upsilon_{-\theta_{t+1}} \left(f(\state_t, \command_t, \noisemodel) - f(\reference \state_t, \reference \command_t, 0)\right) \\
        = &\Upsilon_{-\theta_{t+1}}\frac{\partial f}{\partial \state}\bigg|_{\reference \state_t, \reference \command_t, 0} \disttoref \state_t\\
        &+ \Upsilon_{-\theta_{t+1}} \frac{\partial f}{\partial \command}\bigg|_{\reference \state_t, \reference \command_t, 0} \disttoref \command_t \\
        &+ \Upsilon_{-\theta_{t+1}} \frac{\partial f}{\partial \noisemodel}\bigg|_{\reference \state_t, \reference \command_t, 0} \noisemodel \\
        \approx &A^*_t \local{\disttoref \state_t} + B \disttoref \command_t + B \noisemodel \numberthis \label{eq:invariant_controller}
\end{align*}
with:
$$\reference A_t = \begin{pmatrix}
            1 & \tau \reference \omega_t & 0 \\
            -\tau \reference \omega_t & 1 & \tau \reference u_t \\
            0 & 0 & 1
\end{pmatrix}$$
Once again, the linearized matrices of the controller do not depend on the reference path $(\reference \state_t)_{t=1 \ldots n}$, as $\reference A_t$ depends only on the inputs.
\begin{proposition}
    The linearized equation of the invariant LQ controller in the Frénet coordinates is $\local{\disttoref \state}_{t+1} = A^*_t \local{\disttoref \state_t} + B \disttoref \command_t + B \noisemodel$ where $B$ is constant and $A^*_t$ only depends on the inputs. 
\end{proposition}

We can then apply a LQ control policy to this linearized system, which minimizes the quadratic cost function $J(\state, \command):$
\begin{equation}
    J(\state, \command) = \mathbb E\left(\sum_{t=0}^n (\local{\disttoref \state_t} C (\local{\disttoref \state_t})^T
    + \disttoref \command_t D \disttoref \command_t^T) \right)
    \label{eq:inv_cost}
\end{equation}
under the constraints \eqref{eq:invariant_controller}.
The minimal cost is obtained for $\disttoref \command_t = L^{inv}_t \local{\disttoref \state_t}$ with $L^{inv}_t$ given by the solution to the backwards Riccati equation \eqref{eq:LQ_Riccati_gain_inv}:
\begin{align}
    S_l &= C \\
    L^{inv}_t &= -(B^T S_t B + D)^{-1} B^T S_t \reference A_t \\ 
    S_{t} &= C + A_{t+1}^{*T} S_{t+1} \reference A_{t+1} + A_{t+1}^{*T} S_{t+1} B L^{inv}_{t+1} 
    \label{eq:LQ_Riccati_gain_inv}
\end{align}
However, in presence of measurement uncertainty, the true state $\state$ is unknown. Consequently, the control policy applied in practice is:
\begin{equation}
    \disttoref \command_t = L^{inv}_t \Upsilon_{-\reference \theta_t}(\estimate \state_t - \reference \state_t)
    \label{eq:control_inv}
\end{equation}
%
%

\subsection{Invariant LQG}
Finally, the invariant Kalman filter and the invariant LQ can be combined, in order to compute on-line the (approximate) best input given the current estimation, covariance and the latest input and measurement.
The algorithm steps are summarized in Algorithm \ref{alg:inv_LQG}.
\begin{algorithm}[h]
    \caption{Invariant LQG}
    \label{alg:inv_LQG}
    \begin{algorithmic}
        \REQUIRE Reference trajectory $\left( (\state_t, \command_t) \right)_{t=1\dots n}$
        \REQUIRE Initial covariance $P_0$
        \REQUIRE Off-line calculations of the Riccati gains \eqref{eq:LQ_Riccati_gain_inv}
        \STATE{$\command_0 \leftarrow \reference \command_0$}
        \STATE{$\estimate \state_0 \leftarrow \reference \state_0$}
        \FOR{$0 < t \leq n$}
        \STATE{
            \begin{itemize}
                \item propagate estimation and covariance with \eqref{eq:Kalman_pred_est_inv} and \eqref{eq:Kalman_pred_cov_inv} \\
                \item acquire measurement $\measurement_t$\\
            \item update the best estimate and its covariance using \eqref{eq:invariant_estimate} and \eqref{eq:Kalman_updt_cov_inv}\\
                \item output new command $\command_t$ computed by \eqref{eq:control_inv}
            \end{itemize}
        }
        \ENDFOR
    \end{algorithmic}
\end{algorithm}

\section{Illustration of the robustness property through extensive simulations}
\label{sec:simulation_LQG}
Beyond the fact that it is natural to use a closed-loop control that does not depend upon a non-trivial choice of frame orientation, the IEKF is known to have some convergence guaranteed properties (see \cite{bonnabel_cdc07,bonnabel2009invariant, arxiv-08}) about trajectories defined by constant inputs: indeed for fixed $\reference u_t,\reference \omega_t$ we see that the linearized observation and control systems \eqref{eq:invariant_observer} and \eqref{eq:invariant_controller} become time-invariant, leading to convergence of the gain matrices. However for arbitrary reference trajectories on the one hand,  and large noises that potentially make the observer-controller step out of the  region where the linearization is valid on the other hand, the robustness of the IEKF has never been proved. In this section we show through simulations how the invariant LQG can exhibit increased robustness to noise and initial uncertainties compared to the conventional LQG.
The simulations were performed using a reference trajectory composed of straight lines and curves, displayed on Figure \ref{fig:lost}.

We considered a reference initial covariance $P^0_0$ and reference model and measurement noise covariances $M^0$ and $N^0$.
We compared the performances of the invariant LQG and the conventional LQG by performing several simulations with initial covariance $\alpha^2P^0_0$ and noises covariances $\beta^2M^0, \beta^2N^0$, for various factors $(\alpha^2, \beta^2)$.
For each simulation $(\alpha^2, \beta^2)$ is fixed and we draw 5,000 random initial positions $(x^0_i)_{i=1 \ldots 5000}$ and 5,000 noise samples $((\noisemodel_0, \ldots,\noisemodel_n)_i, (\noisemeas_0, \ldots,\noisemeas_n)_i)_{i=1 \ldots 5000}$, 
Each sample $i$ is used to simulate one robot trajectory using the invariant LQG observer-controller and one robot trajectory using the conventional LQG.
In total we have 10,000 simulated trajectories, half of them using invariant LQG, and the other half using conventional LQG.
For each simulated trajectory we evaluate the cost $I$ as defined in \eqref{eq:cost}.
The results are displayed on Figure \ref{fig:LQGcomparison}, on which we plot:
\begin{enumerate}
    \item The mean costs for invariant and conventional LQG in function of the noise factors $(\alpha^2, \beta^2)$. For each noise factors couple, we also indicate above the bars and between parentheses the percentage of draws for which using the invariant LQG leads to a lower trajectory cost.
    \item The number of trajectories that we can consider as \enquote{lost}. A trajectory is considered \enquote{lost} when the Mahalanobis distance between its final state and the final estimate exceeds a given threshold.
        The retained criterion used to label a trajectory as \enquote{lost} is: 
        \begin{equation}
            (x_n - \estimate x_n, y_n - \estimate y_n)
        P_{n_{[1:2,1:2]}}^{-1} \left( \begin{smallmatrix} x_n - \estimate x_n \\ y_n - \estimate y_n \end{smallmatrix} \right) > F_2^{-1}(0.999)
            \label{eq:lost}
        \end{equation}
        This Mahalanobis distance under the linear and Gaussian assumption follows a $\chi_2^2$ distribution with $2$ degrees of freedom, and  $F_2$ denotes its cumulative distribution function, so that the threshold should not be exceeded for  $99,9\%$ of the trajectories on average.
        In practice a trajectory is \enquote{lost} when the robot's state went outside of the \enquote{tube} inside which the linearization is valid.
        In this case the LQG control cannot be trusted anymore.
\end{enumerate}
\begin{figure}[h]
    \centering
    \includegraphics[width=\linewidth]{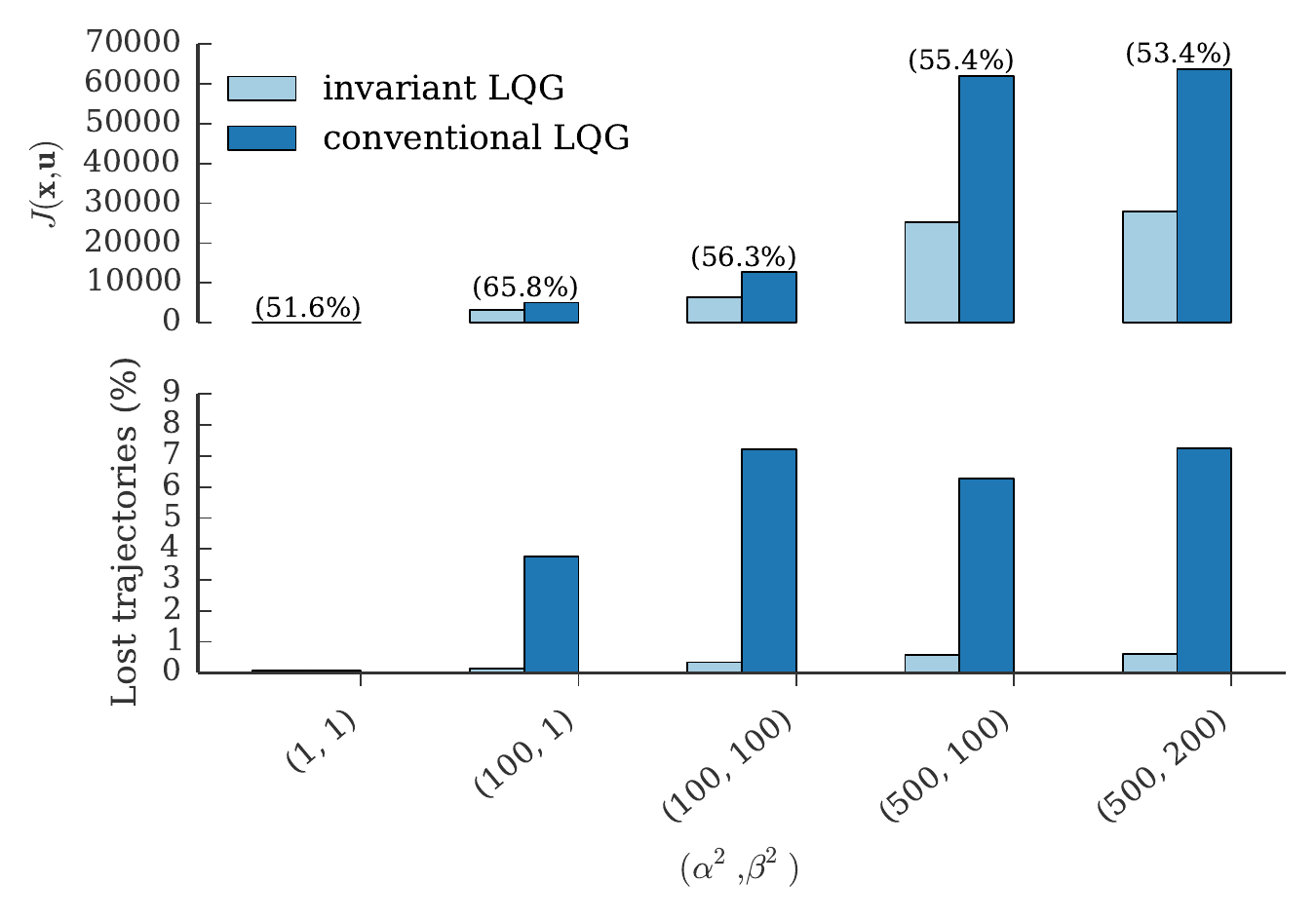}
    \caption{(above) The empirical cost $J(\state, \command)$ obtained for the invariant LQG is up to twice better than the corresponding conventional LQG cost in case of high noises. The percentages between parentheses indicate the proportion of samples for which the invariant LQG leads to a lower cost. (below) The number of \enquote{lost} trajectories (according to criterion \eqref{eq:lost}) is greatly reduced when using the invariant LQG in the case of high noises.}
    \label{fig:LQGcomparison}
\end{figure}

We can draw two main conclusions from these simulations.
First the trajectory cost is on average lower when the invariant LQG is used.
This becomes more significant as the initial covariance increases.
When the initial covariance is very high ($\alpha^2 \geq 100$), the mean cost of invariant LQG trajectories is about twice lower than the corresponding conventional LQG mean cost.
The percentage of samples for which the invariant LQG trajectory has a lower cost than its conventional LQG counterpart also increases with $\alpha^2$.
The influence of the noise levels $\beta^2$ is less significant.

The second bar chart of  Fig. \ref{fig:LQGcomparison} plots the number of \enquote{lost} trajectories for different initial covariance and noises covariances levels.
The number of \enquote{lost} trajectories increases with both $\alpha^2$ and $\beta^2$.
However, in case of high $\alpha^2$ and $\beta^2$, the invariant LQG observer-controller is much less prone to losing the reference trajectory than the conventional LQG.
These \enquote{lost} trajectories have a very high cost insofar as they are very \enquote{far} from the reference.
Consequently, this explains, at least partly, the gap between the mean costs observed on the first bar chart for high $\alpha^2$ and $\beta^2$.
Figure \ref{fig:lost} shows an example draw for which the invariant LQG manages to follow the reference trajectory, whereas the conventional LQG is completely lost. 
\begin{figure}[h]
    \centering
    \includegraphics{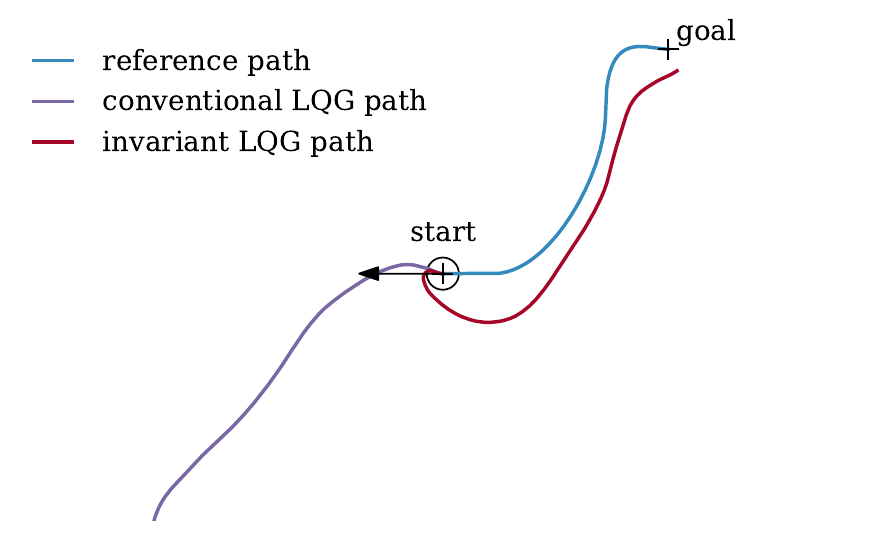}
    \caption{Example of a random draw under strong noises, for which the conventional LQG does not manage to follow the reference trajectory, whereas the invariant LQG does. For this particular draw the initial orientation is opposite to the reference orientation.}
    \label{fig:lost}
\end{figure}
This figure greatly illustrates the increased robustness to uncertain initial condition gained by using the invariant approach.

\section{Consistency of the computed covariance and motion planning}
\label{sec:apriori}

In this section we explore experimentally to what extent the  covariance returned by the observer-controller realistically represents the covariance of the actual discrepancy between the trajectory followed by the robot and a planned trajectory. We prove experimentally that the invariant approach captures more closely the uncertainties than the conventional approach. The main application of those results deal the improvements brought by our methodology for the so-called LQG-MP (motion planning) approach recently introduced in \cite{VanDenBerg11p895}, where the idea is to pick some sensible trajectories based on the uncertainties they convey.   

\subsection{Assessing uncertainty to a planned trajectory}
The idea of LQG-MP \cite{VanDenBerg11p895} is to be able to assess uncertainties to the ability of a closed-loop LQG system to follow various planned trajectories. Many candidate trajectories are generated, and only the trajectories satisfying some criteria are retained (for instance the ones maximizing the probability to reach the goal, or minimizing the probability of collision). A typical example where assessing a level of uncertainty to a planned trajectory may prove useful is displayed on Figure \ref{fig:robust_mp}.
\begin{figure}[h]
    \centering
    \includegraphics[width=0.5\linewidth]{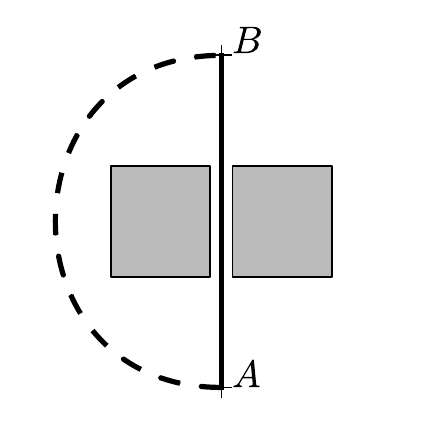}
    \caption{A robust motion planning algorithm will most likely choose the path that bypasses the obstacles (dashed) rather than the one passing in between (plain), in order to avoid any collision when the plan is executed.}
    \label{fig:robust_mp}
\end{figure}

\subsection{A priori probability distributions for an invariant LQG}
\label{subsec:apriori_inv}
In this section, we use the methodology of \cite{VanDenBerg11p895}, to compute the a priori distributions of the state of the robot along a given reference path.
Indeed, the linearization performed in Section \ref{sec:invariant} make possible to analyze in advance how the state of the robot will evolve during the execution of the trajectory if the robot uses an invariant LQG.
The expressions \eqref{eq:invariant_observer} and \eqref{eq:invariant_controller} can be combined to determine, in the Frénet coordinates, the a priori distributions of the state $(\state_t)_{t=1 \dots n}$ around the reference trajectory $(\reference \state_t)_{t=1 \dots n}$ when applying the optimal control $(\disttoref \command_t)_{t=1 \dots n} = (L^{inv}_t \Upsilon_{-\reference \theta_t}(\estimate \state_t - \reference \state_t))_{t=1 \dots n}:$
\begin{align}
    \begin{pmatrix}
        \local{\disttoref{\state}} \\
        \local{\disttoestimate{\state}}
    \end{pmatrix}_{t+1} = &
    \begin{pmatrix}
        \reference{A}_t + BL^{inv}_t & BL^{inv}_t \\
        0 & A_t - K^{inv}_t H A_t
    \end{pmatrix}
    \begin{pmatrix}
        \local{\disttoref{\state}} \\
        \local{\disttoestimate{\state}}
    \end{pmatrix}_t \nonumber \\
    &+ 
    \begin{pmatrix}
        B & 0 \\
        K^{inv}_tHB-B & K^{inv}_t
    \end{pmatrix}
    \begin{pmatrix}
        \noisemodel \\
        \noisemeas_t
    \end{pmatrix}
\end{align}
In the above equation, all the matrices but $A_t$ and $K^{inv}_t$ can be computed in advance, before any actual simulation or measurement.
In order to derive the a priori distributions of the state, we make the approximation $A_t \approx \reference A_t$ and we compute the Kalman gains $K^{inv}_t$ by replacing $A_t$ by $\reference A_t$ in \eqref{eq:invariant_observer}.
\begin{align}
    \begin{pmatrix}
        \local{\disttoref{\state}} \\
        \local{\disttoestimate{\state}}
    \end{pmatrix}_{t+1} = &
    \begin{pmatrix}
        \reference{A}_t + BL^{inv}_t & BL^{inv}_t \\
        0 & \reference{A}_t - K^{inv}_tHA^*_t
    \end{pmatrix}
    \begin{pmatrix}
        \local{\disttoref{\state}} \\
        \local{\disttoestimate{\state}}
    \end{pmatrix}_t \nonumber \\
    &+ 
    \begin{pmatrix}
        B & 0 \\
        K^{inv}_tHB-B & K^{inv}_t
    \end{pmatrix}
    \begin{pmatrix}
        \noisemodel \\
        \noisemeas_t
    \end{pmatrix} \nonumber \\
    = &F_t 
    \begin{pmatrix}
        \local{\disttoref{\state}} \\
        \local{\disttoestimate{\state}}
    \end{pmatrix}_t + G_t \mathbf q_t \\
    \mathbf q_t \sim &\normal(0, Q_t), \quad Q_t =
    \begin{pmatrix}
        M & 0 \\
        0 & N_t
    \end{pmatrix} \nonumber
\end{align}
If the initial state covariance $\mathbb E(\local{\disttoref \state} (\local{\disttoref \state})^T)$ is a known Gaussian of mean $0$ and covariance $P_0,$ and assuming the noises independence, the above formula shows that at each time step $t$, the state's distribution (of the tracking error system) is centered and normal.
Knowing that 
$\mathbb E (\local{\disttoref \state} (\local{\disttoref \state})^T) = P_0,$
$\mathbb E (\disttoestimate \state (\local{\disttoestimate \state})^T) = P_0,$ and
$\mathbb E (\disttoestimate \state (\local{\disttoref \state})^T) = -P_0,$ we can recursively compute the covariance matrices
$\Sigma_t =\mathbb E\left(
    \left(
    \begin{smallmatrix} 
        \local{\disttoref \state} \\
        \local{\disttoestimate \state}
    \end{smallmatrix}
        \right)_t
        \left(
    \begin{smallmatrix} 
        \local{\disttoref \state} \\
        \local{\disttoestimate \state}
    \end{smallmatrix}
    \right)^T_t
\right)$ with the following formula:
\begin{align}
    \Sigma_0 &= \begin{pmatrix}
            P_0 & -P_0 \\
            -P_0 & P_0
          \end{pmatrix} \nonumber \\
    \Sigma_{t+1} &= F_t \Sigma_t F_t^T + G_t Q_t 
    \label{eq:inv_distribution}
\end{align}
The submatrix of $\Sigma_t$ restricted to the three first lines and columns, denoted by ${\Sigma_{[1:3,1:3]}}_t$, is the covariance matrix of the state along the reference path in the Frénet coordinates.
To get the corresponding covariance in the fixed frame, the matrix shall be rotated: $\Upsilon_{\theta_t} {\Sigma_{[1:3,1:3]}}_t \Upsilon_{-\theta_t}$.
Consequently, formula \eqref{eq:inv_distribution}, provides an invariant formulation of the a priori probability distributions of the state about the reference trajectory.

\subsection{A priori probability distributions for a conventional LQG}
\label{subsec:apriori_clas}

The approach advocated in  \cite{VanDenBerg11p895} consists of linearizing both the observer equations and the controller equation about the reference path. For our specific model \eqref{eq:dynamics}, the equations of \cite{VanDenBerg11p895}, become:
\begin{equation}
    \begin{pmatrix}
        \state - \reference \state \\
        \estimate \state - \reference \state
    \end{pmatrix}_{t+1} = F'_t
    \begin{pmatrix}
        \state - \reference \state \\
        \estimate \state - \reference \state
    \end{pmatrix}_{t} + G'_t \mathbf{q}_t
\end{equation}
with:
\begin{align*}
    \quad \mathbf{q}_t &\sim \normal(0, Q_t), \\
    F'_t &= \begin{pmatrix}
              A'_t & B'_t L'_t \\
              K'_t & A'_t + B'_t L'_t - K'_t H A'_t
            \end{pmatrix}, \\
    G'_t &= \begin{pmatrix}
              B'_t & 0 \\
              K'_t H B'_t & K'_t W_t
            \end{pmatrix}
\end{align*}
$$
    A'_t = \begin{pmatrix}
             1 & 0 & -\tau \reference u_t \sin(\reference \theta_t) \\ 
             0 & 1 & \tau \reference u_t \cos(\reference \theta_t) \\ 
             0 & 0 & 1
           \end{pmatrix}, \quad
    B'_t = \begin{pmatrix}
             \tau \cos(\reference \theta_t) & 0 \\
             \tau \sin(\reference \theta_t) & 0 \\
             0 & 1 
           \end{pmatrix}
$$

and $K'_t$ and $L'_t$ are respectively the Kalman and LQ gains for the linearized system:
\begin{equation}
    \label{eq:LQG_linear}
    \disttoref \state_{t+1} = A'_t \disttoref \state_t + B'_t \disttoref \command_t + B'_t \noisemodel
\end{equation}
Finally the prediction for the conventional LQG can be obtained by the following recursion:
\begin{align}
    \Sigma'_0 &= \begin{pmatrix}
                P_0 & 0 \\
                0   & 0
            \end{pmatrix} \nonumber \\
    \Sigma'_{t+1} &= F'_t \Sigma'_t F^{'T}_t + G'_t Q_t G^{'T}_t
    \label{eq:clas_distribution}
\end{align}

As before, the submatrices $3 \times 3$ of $\Sigma'_t$ with indexes inferior to $3$ denoted by ${\Sigma'_{[1:3,1:3]}}_t$ give the state covariance around the reference path.
Contrarily to the invariant formulation, the matrices $B'_t$ depend on the time.
Likewise, the matrices $A'_t$ depend explicitly on the trajectory whereas the corresponding invariant matrices $A_t$ only depend on the inputs.

\subsection{Simulation results}
\begin{figure}[!thb]
    \centering
    \includegraphics[width=\linewidth]{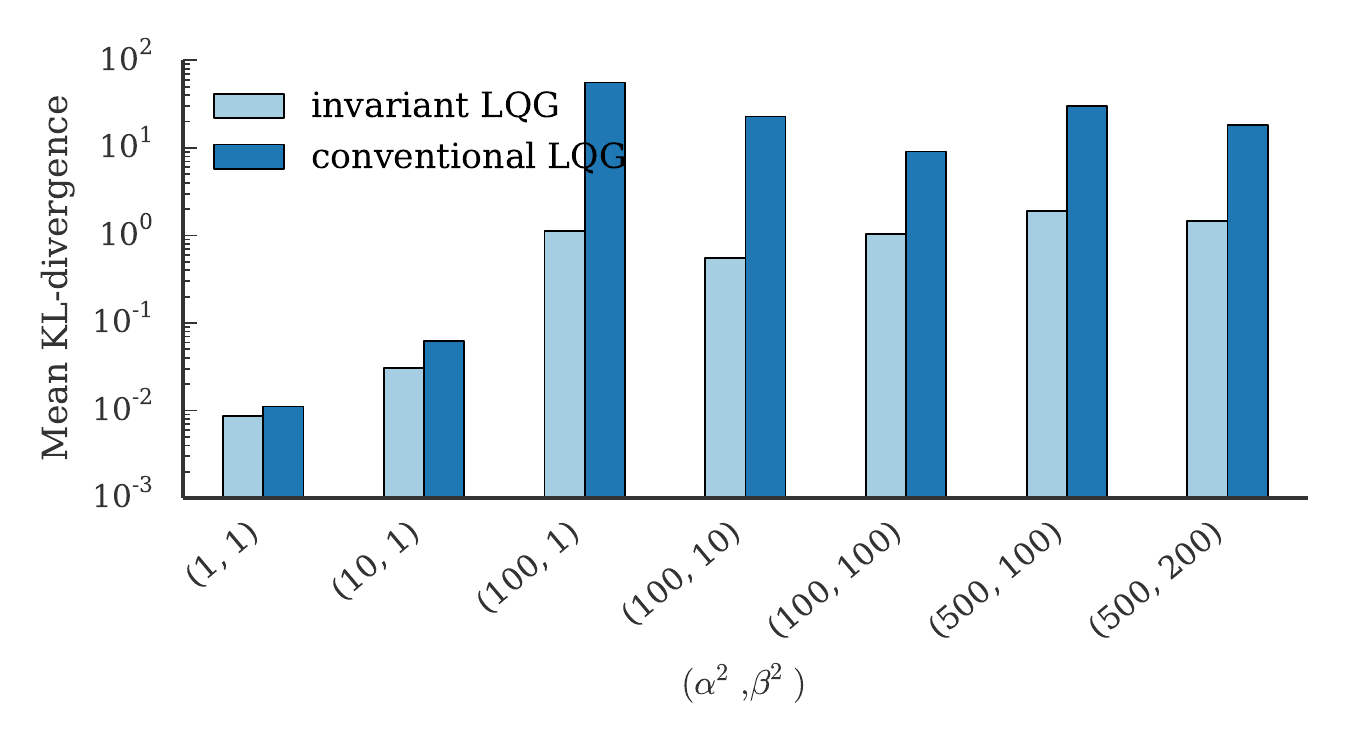}
    \caption{When using the invariant LQG, the predictions of the state distribution better match the simulations (lower KL-divergence). This is increasingly significant as the noises $(\alpha, \beta)$ grow. The KL-divergence is shown using a logarithmic scale.}
    \label{fig:KL}
\end{figure}
In order to compare the prediction performances of the invariant approach (subsection \ref{subsec:apriori_inv}) and the conventional approach (subsection \ref{subsec:apriori_clas}), we reuse the simulations of  Section \ref{sec:invariant}.
For each initial covariance level $\alpha^2$ and noises level $\beta^2$ we can compute, in advance, the predicted covariance matrices $\Sigma_t$ and $\Sigma'_t$  thanks to \eqref{eq:inv_distribution} and \eqref{eq:clas_distribution} respectively.
To measure the \enquote{distance} between the predicted distributions and the actual distributions obtained by simulation, we use the symmetric Kullback-Leibler divergence (KL-divergence) that is a natural way to measure a discrepancy between probability distributions.
For two Gaussians $\normal(\mathbf{m_0}, \Sigma_0)$ and $\normal(\mathbf{m_1}, \Sigma_1)$ of dimension $n$, it is given by:
\begin{align}
    KL = \frac{1}{4} \biggr(&\tr (\Sigma^{-1}_1\Sigma_0) + (\mathbf{m}_1 - \mathbf{m}_0)^T\Sigma^{-1}_1(\mathbf{m}_1 - \mathbf{m}_0) \nonumber \\
        &- \log \frac{\det \Sigma_0}{\det \Sigma_1} - k \biggr) + \nonumber\\
        \frac{1}{4} \biggr(&\tr (\Sigma^{-1}_0\Sigma_1) + (\mathbf{m}_0 - \mathbf{m}_1)^T\Sigma^{-0}_0(\mathbf{m}_0 - \mathbf{m}_1) \nonumber \\
        &- \log \frac{\det \Sigma_1}{\det \Sigma_0} - k\biggr)
\end{align}

The results are displayed on Figure \ref{fig:KL}.
The prediction performance is equivalent for low $\alpha^2$ and $\beta^2$ but the invariant prediction is by far more accurate when these noise factors increase (more than ten times).
In fact, we can see on Figure \ref{fig:prediction} that the predicted covariance matrices $\Upsilon_{\reference \theta_t} {\Sigma_{[1:3,1:3]}}_t \Upsilon_{-\reference \theta_t}$ and ${\Sigma'_{[1:3,1:3]}}_t$ are very close even for large noises when compared in the same frame. This is no surprise: about the reference trajectory as long as the linear approximation is valid the frame in which the equations are derived should not matter that much.  However, when moving away from the reference trajectory, as in actual experiments, the non-linearities may play an important role, and the nice non-linear structure of the invariant LQG saves the day:  as shown in Section \ref{sec:invariant}, the invariant LQG is much more robust to high noise factors, while a non negligible number of conventional LQG trajectories get lost and their behavior becomes random. This results in very high divergences for the conventional prediction while the invariant prediction is still accurate.

\begin{figure}[!thb]
    \centering
    \includegraphics[width=\linewidth]{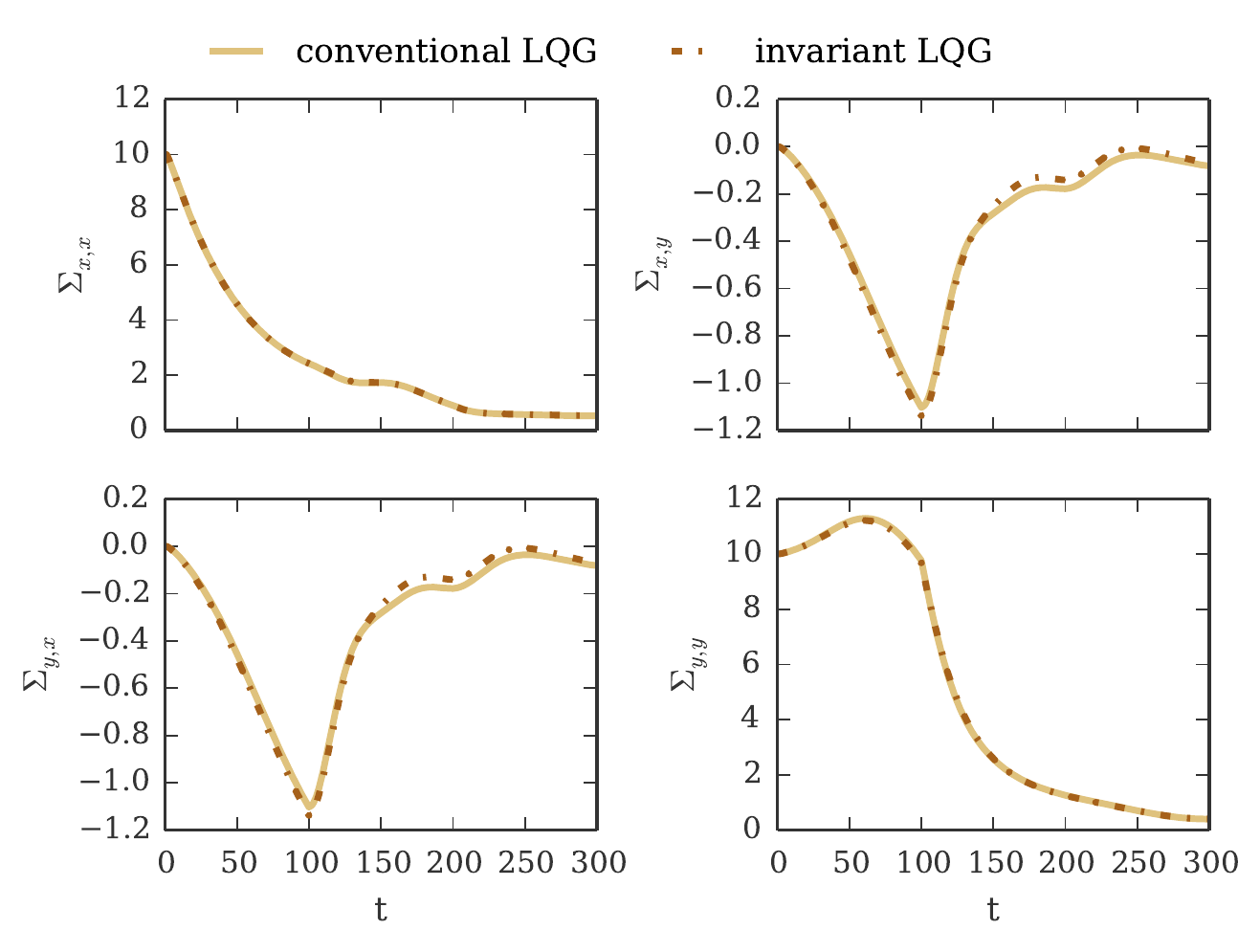}
    \caption{The evolutions over time of the invariant and conventional entries of the predicted covariance matrices are very close when rotated and compared in a common frame. The plot represents the matrices entries in the fixed frame for $(\alpha^2, \beta^2) = (100, 100)$}
    \label{fig:prediction}
\end{figure}
\section{Conclusion}
We introduced a new invariant Linear Quadratic Gaussian controller for the control of a unicycle robot along a reference trajectory.
We showed through extensive simulations that, when noises are strong, the achieved  cost reflecting the magnitude of the tracking error is greatly reduced in comparison to the one obtained when using a conventional LQG.
In practice, the invariant LQG showed increased robustness to high noises, suggesting that the linearized equation of both the observer and the controller have a much higher \enquote{validity zone} than in the conventional LQG case.
The trajectory cost for small noises is comparable, yet slightly better, than the one obtained using a conventional LQG.
Consequently, we recommend the use of invariant LQG over conventional LQG in any application where the initial uncertainty and the model and measurement noises might be high and where symmetries can be exploited.
In the future we would like to illustrate the results through real experimentations, and would also like to explore the superiority of the invariant approach from a theoretical viewpoint.

\printbibliography
\end{document}